\documentclass[letterpaper]{article} 
\usepackage{aaai2026}  
\usepackage{times}  
\usepackage{helvet}  
\usepackage{courier}  
\usepackage[hyphens]{url}  
\usepackage{graphicx} 
\usepackage{natbib}  
\usepackage{caption} 
\usepackage{graphicx}
\usepackage{latexsym}
\usepackage{xspace}
\usepackage{amsmath}
\usepackage{amsfonts}
\usepackage{xcolor}
\usepackage{mathtools}
\usepackage{subcaption}
\usepackage{array}
\usepackage{amsthm}
\usepackage{amssymb}
\usepackage{multirow}
\usepackage{makecell}
\newcommand{\marianela}[1]{\textcolor{orange}{#1}}

\newcommand{\fluents}{\ensuremath{{\cal F}}\xspace}
\newcommand{\actions}{\ensuremath{{\cal A}}\xspace}
\newcommand{\init}{\ensuremath{{\cal I}}\xspace}
\newcommand{\goal}{\ensuremath{{\cal G}}\xspace}
\newcommand{\cost}{\ensuremath{c}\xspace}

\newcommand{\task}{\ensuremath{{\cal P}}\xspace}
\newcommand{\plan}{\ensuremath{\pi}\xspace}

\newcommand{\stripstask}{\ensuremath{\task=\langle \fluents, \actions, \init, \goal \rangle}\xspace}
\newcommand{\state}{\ensuremath{s}\xspace}

\newcommand{\action}{\ensuremath{a}\xspace}
\newcommand{\name}{\ensuremath{\textsc{name}}\xspace}
\newcommand{\precondition}{\ensuremath{\textsc{pre}}\xspace}

\newcommand{\addeffects}{\ensuremath{\textsc{add}}\xspace}
\newcommand{\deleffects}{\ensuremath{\textsc{del}}\xspace}
\newcommand{\actionapplication}{\ensuremath{\gamma}\xspace}
\newcommand{\planapplication}{\ensuremath{\Gamma}\xspace}

\newcommand{\atpackage}[3]{\textsf{at}(\mathsf{#1\mbox{-}#2}\mbox{ }\mathsf{#3})}
\newcommand{\at}[2]{\textsf{at}(\mathsf{#1}\mbox{ }\mathsf{#2})}

\newcommand{\drivetruck}[2]{\textsf{drivetruck}\mbox{ }\mathsf{#1}\mbox{ }\mathsf{#2}}
\newcommand{\loadtruck}[2]{\textsf{loadtruck}\mbox{ }\mathsf{#1}\mbox{-}\mathsf{#2}}
\newcommand{\unloadtruck}[2]{\textsf{unloadtruck}\mbox{ }\mathsf{#1}\mbox{-}\mathsf{#2}}

\newtheorem{definition}{Definition}
\newtheorem{proposition}{Proposition}

\newtheorem{remark}{Remark}

\newcommand{\lazytask}{\task_{\mathsf{L}}}
\newcommand{\lazyfluents}{ \fluents_{\mathsf{L}}}
\newcommand{\lazyactions}{\actions_{\mathsf{L}}}
\newcommand{\lazyinit}{\init_{\mathsf{L}}}
\newcommand{\lazygoal}{\goal_{\mathsf{L}}}

\newcommand{\eagertask}{\task_{\mathsf{E}}}

\newcommand{\plandisruption}{\ensuremath{{\cal D}}\xspace}
\usepackage{algorithm}
\usepackage{algorithmic}

\usepackage{newfloat}
\usepackage{listings}
\DeclareCaptionStyle{ruled}{labelfont=normalfont,labelsep=colon,strut=off} 
\floatstyle{ruled}
\newfloat{listing}{tb}{lst}{}
\floatname{listing}{Listing}
\title{Planning with Minimal Disruption}
\author{%
Alberto Pozanco,
Marianela Morales,
Daniel Borrajo,
Manuela Veloso}
\affiliations{
J.P. Morgan AI Research\\
\{alberto.pozancolancho,marianela.moraleselena\}@jpmorgan.com, \{name.surname\}@jpmorgan.com}

\usepackage{bibentry}

\begin{document}

\maketitle

\begin{abstract}
In many planning applications, we might be interested in finding plans that minimally modify the initial state to achieve the goals.
We refer to this concept as plan disruption. In this paper, we formally introduce it, and define various planning-based compilations that aim to jointly optimize both the sum of action costs and plan disruption.
Experimental results in different benchmarks show that the reformulated task can be effectively solved in practice to generate plans that balance both objectives.
\end{abstract}

%

\section{Introduction}
Classical planning is the task of finding a plan, which is a sequence of deterministic actions that, when executed from a given initial state, lead to a desired goal state~\cite{DBLP:books/daglib/0014222}.
Each action is associated with a non-negative cost, and the total cost of a plan is defined as the sum of the costs of its actions. Plans with minimal cost are called optimal, and how to efficiently compute them accounts for a large part of automated planning research.
However, the real-world is full of applications where the sum of action costs is only one of the objectives that define the quality of a plan~\cite{sobrinho2001algebra,geisser2022admissible,salzman2023heuristic}.

In this paper, we present an objective that could be significant in various planning applications: the number of modifications required to transform the initial state into the goal state.
We refer to this concept as \emph{plan disruption}, and minimizing it leads to plans that retain as much of the initial state as possible while still achieving the goals.
This property is often desired when computing solutions in related areas such as scheduling and optimization~\cite{sakkout2000probe,mutze2014scheduling}.
Concrete real-world examples include employee scheduling~\cite{clark2011nurse,medard2007airline} or project management~\cite{zhu2005disruption}. 
This concept is closely related to the idea of \emph{avoiding side effects}~\cite{amodei2016concrete}, as it assumes that changes in the state which are not specified in the goal are undesirable and result from underspecification of objectives.

\cite{klassen2022planning} formalized the concept of avoding side effects in planning and introduced a compilation to solve problems by minimizing the number of changes or side effects required to reach the goal state, without considering the quality of the plan used to achieve it.
Also in the field of planning, some studies have explored the related problem of \emph{plan stability}, which focuses on making minimal changes to an existing plan during replanning~\cite{van2005plan,fox2006plan}.
We argue that in certain situations, such as when iteratively solving successive planning tasks within the same environment, minimizing plan disruption could be highly beneficial.
In such scenarios, we may not have an existing plan to adhere to, but instead, we might be interested in computing a plan from scratch while making minimal changes to the initial state, as it includes certain elements that we wish to preserve.

\begin{figure}[t]
    \centering
    \includegraphics[width=0.8\linewidth]{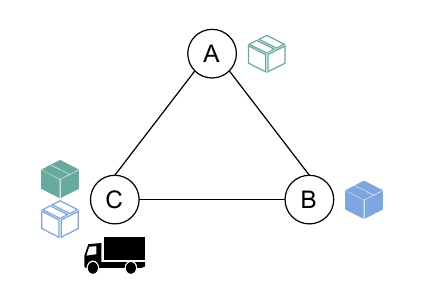}
    \caption{Logistics task where a truck must deliver two packages by moving them from their current (filled) locations to their goal (empty) destinations.}
    \label{fig:logistics}
\end{figure}

Let us consider the logistics task depicted in Figure~\ref{fig:logistics}, where a truck must deliver two packages by moving them from their current (filled) locations to their goal (empty) destinations.
Among the various plans that solve this task, there are two cost-optimal plans (with the minimum number of actions) that we would like to highlight. 
The first plan involves the truck traveling to $\mathsf{B}$, loading the blue package, unloading it at $\mathsf{C}$, then loading the green package, and finally unloading it at $\mathsf{A}$. 
This plan successfully delivers all the packages but leaves the truck at $\mathsf{A}$.
Conversely, there is another cost-optimal plan where the truck first loads the green package, delivers it to $\mathsf{A}$, picks up the blue package at $\mathsf{B}$, and finally delivers it to $\mathsf{C}$. 
This plan also completes the delivery of all packages but leaves the truck at its original location, $\mathsf{C}$, which may be advantageous if $\mathsf{C}$ serves as a depot or a strategic point for the truck's positioning for future planning tasks~\cite{pozanco2024computing}.

In simple examples like this, such constraints can sometimes be encoded as part of the goal state. 
However, this approach is not generally applicable, especially for larger tasks. 
In these cases, it is impossible to determine in advance which aspects of the initial state (such as which trucks) can retain their initial truth values without making the planning task unsolvable. 
A better strategy is to allow the planner to minimize the number of changes performed, ensuring that solutions will always exist.

%

We are interested in jointly optimizing both objectives: sum of action costs and plan disruption.
This contrasts with~\cite{klassen2022planning}, which focuses execlusively on the plan disruption (side effects), disregarding the inherent trade-offs between these two objectives.
There exist three main approaches to solve such multi-objective planning problems in the literature: cost-algebraic A*~\cite{edelkamp2005cost}; multi-objective search algorithms such as NAMOA*~\cite{10.1145/1754399.1754400} or BOA*~\cite{ulloa2020simple,hernandez2023simple}; or reformulate the original planning task so that plans that solve the new task are plans that optimize the two objectives~\cite{katz2022producing,pozanco2024uniform}.
While the first two approaches require developing new heuristics for each metric, reformulating the task allows us to leverage all the power of domain-independent heuristics and planners.
We will reformulate the original task to generate plans that jointly optimize sum of action costs and plan disruption.

The main contributions of this paper are: (1) introduction of a novel bi-objective planning task with many real-world applications: jointly optimizing sum of action costs and plan disruption; and (2) definition of different compilations to produce plans that jointly optimize both objectives.

\section{Background}
We formally define a planning task as follows:
\begin{definition}[Planning task]\label{def:strips-plan-task}
  A Planning task
is a tuple \stripstask, where \fluents is a set of fluents, \actions is a set of
 actions, $\init \subseteq \fluents$ is an initial state, and $\goal\subseteq \fluents$ is a goal specification.  
\end{definition}


A state $\state \subseteq \fluents$ is a set of fluents that are true at a given time.
A state $\state \subseteq \fluents$ is a goal state if and only if $\goal \subseteq \state$.
Each action $\action \in \actions$ is described by its name $\name(\action)$, which is a string; a set of positive and negative preconditions $\precondition^+(\action)$ and $\precondition^-(\action)$, which are set of fluents that need to be true (or false) for the action to be applied; add and delete effects $\addeffects(\action)$ and $\deleffects(\action)$, which are set of fluents that are added (or deleted) once the action is applied; and cost $\cost(\action) \in \mathbb{R}$.
An action \action is applicable in a state \state if and only if $\precondition^+(\action)\subseteq~s$ and $\precondition^-(\action) \cap s = \emptyset$.
We define the result of applying an action in a state as $\actionapplication(\state,\action)=(\state \setminus \deleffects(\action)) \cup \addeffects(\action)$.
We assume $\deleffects(\action) \cap \addeffects(\action) = \emptyset$.
A sequence of actions $\plan=(\action_1,\ldots,\action_n)$ is applicable in a state $\state_0$ if there are states $(\state_1.\ldots,\state_n)$ such that $\action_i$ is applicable in $\state_{i-1}$ and $\state_i=\actionapplication(\state_{i-1},\action_i)$.
The resulting state after applying a sequence of actions is $\planapplication(\state,\pi)=\state_n$, and $\cost(\plan) = \sum_{i}^n \cost(\action_i)$ denotes the cost of $\plan$.
A state $\state$ is reachable from state $\state^\prime$ if and only if there exists an applicable action sequence \plan such that $s \subseteq \planapplication(\state^\prime,\pi)$.
A state \state is a dead-end state if and only if it is not a goal state and no goal state is reachable from \state.
The solution to a planning task $\task$ is a plan, i.e., a sequence of actions $\plan$ such that $\goal \subseteq \planapplication(\init,\pi)$.
A plan with minimal cost is optimal.
\section{Plan Disruption}

In classical planning, a plan \plan solves a planning task if and only if all the goals are true in the final state $\goal \subseteq \planapplication(\init,\plan)$.
This definition overlooks the number of modifications to transform $\init$ into $\goal$.
In this section, we introduce plan disruption, a metric that counts the number of propositions that differ between $\init$ and $\planapplication(\init,\plan)$.
Specifically, the disruption of a plan is calculated as the cardinality of the symmetric difference of the initial and goal state sets. 
The symmetric difference $\triangle$ of two sets is an operation that returns a set that includes elements present in either of the two sets but absent in their intersection.
This definition is similar to the concept of fluent side effect in \cite{klassen2022planning}.

\begin{definition}[Plan Disruption]
    Given a planning task $\stripstask$ and a plan $\plan$ that solves it, the plan disruption of $\plan$ is defined as:
    \begin{equation*}
        \plandisruption(\plan) = |\init \triangle \planapplication(\init,\plan)|
    \end{equation*}
\end{definition}

Let us illustrate Plan Disruption by using the \textsc{logistics} task depicted in Figure~\ref{fig:logistics}.
Table~\ref{tab:logistics} describes the initial state of the task $\init$, together with the two cost-optimal plans outlined in the Introduction.
The first plan $\plan_1$ involves the truck traveling to $\mathsf{B}$, loading the blue package, unloading it at $\mathsf{C}$, then loading the green package, and finally unloading it at $\mathsf{A}$. 
In the second plan $\plan_2$, the truck first loads the green package, delivers it to $\mathsf{A}$, picks up the blue package at $\mathsf{B}$, and finally delivers it to $\mathsf{C}$. 
The third row of Table 1 shows the goal state after executing each respective plan, while the final row presents the plan disruption metric for each plan.
After executing $\plan_1$, none of the fluents that were true in the initial state remain true in the goal state. Instead, three new fluents become true in the goal state that were not true initially, leading to a plan disruption of $6$.
In contrast, plan $\plan_2$ keeps the truck at its original location $C$, altering only the fluents related to the packages' positions. This results in a plan disruption value of $4$.

\begin{table}[t]
\setlength{\tabcolsep}{1.8pt}
\scalebox{0.79}{
    \begin{tabular}{c|>{\centering\arraybackslash}p{0.51\linewidth}|>{\centering\arraybackslash}p{0.51\linewidth}}
          & $\plan_1$ & $\plan_2$  \\
         \hline
          $\init$ & \multicolumn{2}{c}{$\at{green}{C}, \at{blue}{B}, \at{truck}{C}$} \\
         \hline
          $\planapplication(\init,\plan)$ & \scalebox{0.85}{$\at{green}{A}, \at{blue}{C}, \at{truck}{A}$} & \scalebox{0.85}{$\at{green}{A}, \at{blue}{C}, \at{truck}{C}$} \\
         \hline
          \plandisruption(\plan) & $6$ & $4$ \\
          \hline
    \end{tabular}}
    \vspace{1.5mm}
    \caption{Plans $\plan_1$ and $\plan_2$ to reach \goal in the \textsc{logistics} task shown in Figure~\ref{fig:logistics}. Rows in the table define the initial state $\init$, the goal state reached after executing each plan $\Gamma(\init,\plan)$, and the plan disruption $\plandisruption(\plan)$ of each plan.} 
    \label{tab:logistics}
    \vspace{3mm}
\end{table}

We can determine both a lower and an upper bound for the plan disruption of a planning task without actually computing a plan. 
The lower bound is given by the minimum changes required to transition from $\init$ to $\goal$.
On the other hand, the upper bound corresponds to the maximum number of changes that could occur when going from $\init$ to $\goal$. This is given by the total number of fluents $\fluents$ in the planning task, excluding those goal fluents that are already true in the initial state, as they do not need to change.
We provide these bounds to offer theoretical insight into the possible values of plan disruption, and formalize them as follows:

\begin{proposition}
    Given a planning task $\stripstask$, the lower bound $\ell$ and the upper bound $\mathcal{L}$ of the plan disruption $\plandisruption(\plan)$ of any plan $\plan$ that solves $\task$ are defined as: 
    \[ 
    \ell(\plandisruption(\plan))= |\goal \setminus \init|
    \hspace{6mm}
    \mathcal{L}(\plandisruption(\plan)) = |\fluents| - |\goal \cap \init|
    \]
\end{proposition}

\begin{proof}
To establish the lower bound $\ell(\plandisruption(\plan))$, consider the initial state $\init$ and the goal state $\goal$. 
The lower bound is determined by the fluents that are in the goal state but not in the initial state, i.e., $|\goal \setminus \init|$. This represents the minimum number of changes required to achieve the goal state from the initial state, as these fluents must be made true by any plan that solves the task.

For the upper bound $\mathcal{L}(\plandisruption(\plan))$, any fluent in $\{\goal \cap \init\}$ is already true in $\init$ and must also be true in the final state $\planapplication(\init,\plan)$ of any plan $\plan$. Thus, such fluents cannot differ between $\init$ and $\planapplication(\init,\plan)$ and never contribute to disruption. At most, every other fluent can differ once relative to $\init$, and therefore $\mathcal{L}(\plandisruption(\plan)) = |\fluents| - |\goal \cap \init|$. 
This ensures that $\mathcal{L}(\plandisruption(\plan))$ is indeed an upper bound, as it represents the maximum number of changes possible from the initial state.
\end{proof}
\section{Computing Plans that Minimize Plan Disruption}

We present two different compilations that optimize plan disruption.
The compilations balance the trade-off between the accuracy of optimizing plan disruption and the complexity of the reformulated task.
In all cases, we will use $\omega \in \mathbb{R}^{+}_{0}$ to represent the weight assigned to the importance of plan disruption in the quality (cost) of a plan.
Next, we present these compilations, starting with the most accurate and complex, and progressing to the least accurate and simplest.

\subsection{Compilation 1: Lazy Plan Disruption}

The first compilation seeks to minimize Plan Disruption by examining, after the goals have been achieved, which fluents have changed their values from the initial to the goal state.
This compilation follows a similar structure to the soft goals compilation by Keyder and Geffner (\citeyear{keyder2009soft}) to solve oversubscription planning tasks~\cite{smith2004choosing}, which also inspired the compilation in \cite{klassen2022planning} to avoid side effects.
The plan is divided into two parts.
The first part focuses on achieving the hard goals, i.e., the original goals of the problem $\goal$. 
The second part addresses the soft goals by imposing penalties (increasing the cost of the plan) for not achieving them.
In this case the soft goals involve maintaining the truth values of the fluents in the goal state as they were in the initial state.

Given a planning task $\task$, we extend the set of fluents $\fluents$ as follows:
\begin{itemize}
    \item $\fluents_{I} = \bigcup_{f \in \init} \{ \mathsf{init}_f \}$, a set of propositions that mark whether a fluent $f \in \init$ was true in the initial state.

    \item $\fluents_{c} = \bigcup_{f \in \fluents} \{ \mathsf{checked}_f\}$, a set of propositions that mark whether a fluent $f \in \fluents$ has been checked or not. As seen later, this check will test the truth value of $f$ at the end of planning, increasing the total cost of the plan depending on its truth value in \init.

    \item $\mathsf{ga}$, a proposition representing that all goals in \goal have been achieved.

    \item $\mathsf{end}$, a proposition that represents the end of the planning episode. 
\end{itemize}

We update the actions in \actions to force that the original actions are only executed before the goals have been achieved. 
For each action $\action \in \actions$, we generate a modified version $\action^\prime$ that is stored in a new set $\actions^\prime$.
Each action $\action^\prime$ is defined as follows:
\begin{itemize}
    \item $\name(\action^\prime)=\name(\action)$
    \item $\precondition^+(\action^\prime)=\precondition^+(\action)$
    \item  $\precondition^-(\action^\prime)=\precondition^-(\action) \cup \{ \mathsf{ga} \}$
    \item $\addeffects(\action^\prime)=\addeffects(\action)$
    \item $\deleffects(\action^\prime)=\deleffects(\action)$
    \item $\cost(\action^\prime)=\cost(\action)$
\end{itemize}

Apart from updating the original actions, we also extend \actions with new actions that can only be executed after the goals have been achieved.
The first new action $\action^G$ sets $\mathsf{ga}$ to true once a goal state is reached:
\begin{itemize}
    \item $\name(\action^G) = \mathsf{goalstate}$
    \item $\precondition^+(\action^G)=\goal$
    \item $\precondition^-(\action^G)=\emptyset$
    \item $\addeffects(\action^G) = \{\mathsf{ga}\}$
    \item $\deleffects(\action^G) = \emptyset$
    \item $\cost(\action^G)=0$
\end{itemize}

The second set of new actions will increase the total cost of the plan depending on the number of changes between the initial and goal states.
In particular, we generate two actions for each fluent $f \in \fluents$: a $\mathsf{collect}^f$ action that does not increase the total cost and can be executed if and only if $f$ has the same truth value in both states; and a $\mathsf{forgo}^f$ action that increases the total cost by $\omega$ when the truth value of $f$ has changed.
We store these actions in a new set $\actions^c$.
Below we show the $\mathsf{forgo}$ and $\mathsf{collect}$ actions.
\begin{itemize}
    \item $\name(\action^{f})\!=\!\mathsf{forgo}^{f}$
    \item $\precondition^+(\action^f)\!=\!\{ \mathsf{ga}\}$
    \item $\precondition^-(\action^f)\!=\!\{ \mathsf{checked}_f,\!\mathsf{end} \}$
    \item $\addeffects(\action^f)\!=\!\{ \mathsf{checked}_f \} $ 
    \item $\deleffects(\action^f)\!=\!\emptyset$
    \item $\cost(\action^f)\!=\!\omega$
\end{itemize}
We only show the $\mathsf{collect}$ action for the positive case, i.e., when $\mathsf{init}_f$ is true. 
The other $\mathsf{collect}$ case is defined analogously by adding $f$ and $\mathsf{init}_f$ to $\precondition^-$.
\begin{itemize}
    \item $\name(\action^{f})\!=\!\mathsf{collect}^{f}$
    \item $\precondition^+(\action^f)\!=\!\{ f, \mathsf{init}_f, \mathsf{ga}\}$
    \item $\precondition^-(\action^f)\!=\!\{ \mathsf{checked}_f,\!\mathsf{end} \}$
    \item $\addeffects(\action^f)\!=\!\{ \mathsf{checked}_f \} $ 
    \item $\deleffects(\action^f)\!=\!\emptyset$
    \item $\cost(\action^f)\!=\!0$
\end{itemize}

Finally, we extend $\actions$ with a new action $\action^{\mathsf{end}}$ that makes $\mathsf{end}$ true when all the goals have been achieved and the truth value of all the fluents in $\fluents$ have been checked. 
\begin{itemize}
    \item $\name(\action^{\mathsf{end}}) = \mathsf{end}$
    \item $\precondition^+(\action^{\mathsf{end}})=\{\mathsf{ga}\} \cup \bigcup_{f \in \fluents}\{ \mathsf{checked}_f\}$
    \item $\precondition^-(\action^{\mathsf{end}})=\emptyset$
    \item $\addeffects(\action^{\mathsf{end}}) = \{\mathsf{end}\}$
    \item $\deleffects(\action^{\mathsf{end}}) = \emptyset$
    \item $\cost(\action^{\mathsf{end}})=0$
\end{itemize}


\begin{definition} [Lazy Plan Disruption Planning Task]
    Given a planning task $\stripstask$, a lazy plan disruption planning task $\lazytask^\omega$ that optimizes sum of action costs and plan disruption is a tuple $\lazytask^\omega = \langle \lazyfluents, \lazyactions, \lazyinit, \lazygoal \rangle$ where:
    \begin{itemize}
        \item $\lazyfluents = \fluents \cup \fluents_I \cup \fluents_c \cup \{\mathsf{ga}, \mathsf{end}\}$
        \item $\lazyactions = \actions^\prime \cup \actions^c \cup \{\action^G, \action^{\mathsf{end}}\}$
        \item $\lazyinit = \init \cup \fluents_I$
        \item $\lazygoal = \{ \mathsf{end}\}$
    \end{itemize}
\end{definition}

This compilation generates a number of actions that is given by the following formula:
\begin{equation*}
    |\lazyactions| = |\actions^\prime| + \overbrace{(2\times|\fluents|)}^{\actions^c} + \overbrace{2}^{\action^G,\action^{\mathsf{end}}}
\end{equation*}

The completeness and soundness of the lazy plan disruption planning task $\lazytask^\omega$ with respect to the original planning task $\task$ are straightforward derived from Keyder's work~\cite{keyder2009soft}.

\begin{proposition}\label{prop:lazy-plan-disruption}
    Given a planning task $\task$ and a plan $\plan$ that solves it, there exists a plan $\plan'$ mapped from $\plan$ that solves $\lazytask^1$, where the plan disruption of $\plan$, $\plandisruption(\plan)$, is equal to the plan disruption obtained by the lazy compilation, $\plandisruption_{\mathsf{L}}(\plan)$, and is given by:
    \[
    \plandisruption(\plan) = \cost(\plan') - \cost(\plan) = \plandisruption_{\mathsf{L}}(\plan)
    \]
\end{proposition}

\begin{proof}

Observe that $\plan'$ is mapped from $\plan$, i.e., $\plan'$ is obtained by appending to $\plan$ the $\action^{G}$ action followed by some permutation of the actions in $\actions^c$ and finally by the end action $\action^{\mathsf{end}}$. Since the cost of the appended actions is zero except for the $\mathsf{forgo}^{f}$ actions in $\actions^c$, then we have that $\cost(\plan') = \cost(\plan) + \sum_{\mathsf{forgo}^{f}\in \plan'} \cost(\mathsf{forgo}^{f})$. We now have to check that $\sum_{\mathsf{forgo}^{f}\in \plan'} \cost(\mathsf{forgo}^{f}) = \plandisruption(\plan)$.

A $\mathsf{forgo}^{f}$ action is executed if a fluent that was true in the initial state is no longer true in the goal state, and viceversa, increasing the total cost by $\omega = 1$. This behavior mirrors the plan disruption metric, which is defined as the symmetric difference between the initial and goal states.

Therefore, the total cost of the $\mathsf{forgo}^{f}$ actions with $\omega=1$, $\sum_{\mathsf{forgo}^{f} \in \actions^c} \cost(\mathsf{forgo}^{f})$, is equal to the plan disruption $\plandisruption(\plan)$, as it effectively counts the number of fluents that differ between the initial and goal states.
\end{proof}

\subsection{Compilation 2: Eager Plan Disruption}
The previous compilation is capable of accurately replicating the Plan Disruption metric.
However, it has a notable limitation: the quality of a plan, in terms of plan disruption, can only be assessed at the conclusion of the planning process, with no intermediate indicators.
This limitation may lead search algorithms to backtrack multiple times in their quest to identify optimal plans.
To address this issue, we propose a new compilation approach that prioritizes search efficiency over the precise accuracy of plan disruption. 
This is achieved by continuously monitoring the number of changes made to the initial state throughout the planning process. 
However, this monitoring will only compare the current state with \init, which might introduce inaccuracies in the plan disruption metric.
By doing so, we aim to enhance the search algorithm's efficiency, allowing it to more effectively navigate towards optimal plans.

Given a planning task \task, we update each original action $\action \in \actions$ by modifying its cost depending on (1) the number of add and delete effects of $\action$; and (2) the value of the added/deleted fluents in the initial state \init. 
We refer to the set of updated actions as $\actions^{\prime\prime}$, and formally define each action $\action^{\prime}\in\actions^{\prime\prime}$ as follows:
\begin{itemize}
    \item $\name(\action^\prime)=\name(\action)$
    \item $\precondition^+(\action^\prime)=\precondition^+(\action)$
    \item $ \precondition^-(\action^\prime)=\precondition^-(\action)$
    \item $\addeffects(\action^\prime)=\addeffects(\action)$
    \item $ \deleffects(\action^\prime)=\deleffects(\action)$
    \item $\cost(\action^\prime) = \cost(\action) + \omega(|(\addeffects(\action) \setminus \init) \cup (\deleffects(\action) \cap \init)|)$
\end{itemize}
As we can see, the only difference between $\action$ and $\action^\prime$ lies in the cost, which now incorporates an extra term weighted by $\omega$.
This extra term counts the number of propositions that are added by $\action$ and were not present in $\init$, plus the number of propositions that were present in $\init$ and are removed by $\action$.
Since this check is only done with respect to $\init$, it might not accurately measure plan disruption, as seen later.

\begin{definition}[Eager Plan Disruption Planning Task]
     Given a planning task $\stripstask$, an eager plan disruption planning task $\eagertask^\omega$ that optimizes sum of action costs and plan disruption is a tuple $\eagertask^\omega = \langle \fluents, \actions^{\prime\prime}, \init, \goal \rangle$.
\end{definition}

\begin{remark}\label{rem:eager-plan-disruption}
    Observe that the plan disruption of $\plan$ that solves $\task$ approximated by the eager compilation $\eagertask^\omega$ can be obtained as in Proposition~\ref{prop:lazy-plan-disruption}: there exists $\plan''$ mapped from $\plan$ that solves $\eagertask^1$ and the plan disruption of $\plan$ obtained by the eager compilation is given by:
    \[
    \plandisruption_{\mathsf{E}}(\plan) = \cost(\plan'') - \cost(\plan) >= \plandisruption(\plan)
    \]
\end{remark}

\begin{proposition}
    Given a planning task $\task$ and a plan $\plan$ that solves it, the plan disruption obtained by the lazy compilation $\plandisruption_{\mathsf{L}}(\plan)$, is always lower than or equal to the plan disruption obtained by the eager compilation $\plandisruption_{\mathsf{E}}(\plan)$.
\end{proposition}

\begin{proof}
    By Proposition~\ref{prop:lazy-plan-disruption} and Remark~\ref{rem:eager-plan-disruption}, we have that there exist $\plan'$ and $\plan''$ that solve $\lazytask^1$ and $\eagertask^1$ respectively, and the plan disruption obtained by $\lazytask^1$ is $\plandisruption_{\mathsf{L}}(\plan) = \cost(\plan') - \cost(\plan)$ and the one obtained by $\eagertask^1$ is $\plandisruption_{\mathsf{E}}(\plan) = \cost(\plan'') - \cost(\plan)$. Then we have to show that $\cost(\plan')\le\cost(\plan'')$. In particular, $\cost(\plan') = \cost(\plan) + \sum_{\mathsf{forgo}^{f}\in \plan'} \cost(\mathsf{forgo}^{f})$, and $\cost(\plan'') = \cost(\plan) + \sum_{\action\in\plan} |(\addeffects(\action) \setminus \init) \cup (\deleffects(\action) \cap \init)|$. We then have to show that

    \begin{small}
        \[\sum_{\mathsf{forgo}^{f}\in \plan'} \cost(\mathsf{forgo}^{f})\le \sum_{\action\in\plan} |(\addeffects(\action) \setminus \init) \cup (\deleffects(\action) \cap \init)|\]
    \end{small}

    Consider a fluent $f$ that was not true in \init and is added to the state by an action.
    If $f$ 
    is later removed by another action, the eager task $\eagertask^1$ still counts the addition of $f$ even if it is subsequently removed, since $\eagertask^1$ does not account for net changes. The lazy task $\lazytask^1$, however, does not count such transient changes if they do not result in a net effect on the goal state. This results in a lower or equal cost compared to the eager task, leading to the conclusion that the plan disruption for $\lazytask^1$ is always less than or equal to that for $\eagertask^1$.
\end{proof}

\section{Example}

Let us show how the two compilations approximate Plan Disruption by using the simple planning task depicted in Table~\ref{tab:running_example}.

\begin{table}[]
\small
    \centering
    \begin{tabular}{p{4.7cm}|p{2.5cm}}
       $\fluents~=~\{ A,B,C,D\},\init~=~\{ A, B\}$  & $\goal=\{ D \}$  \\ \hline
       $\precondition^+(a_1)~=~\{ A\} $  & $\precondition^+(a_2)~=~\{ C\} $ \\
       $\deleffects(a_1) = \{ A,B\} $  & $\deleffects(a_2) = \{ A\} $ \\
       $\addeffects(a_1) = \{ C\} $  & $\addeffects(a_2)~=~\{ D,B\} $ \\ 
       $\cost(a_1) = 10$  & $\cost(a_2) = 10$ \\\hline
       \multicolumn{2}{c}{$\plan = (a_1, a_2) \Rightarrow (s_0 = \{ A,B\}, s_1 = \{C\}, s_2 = \{ C,D,B\})$}   \\ \hline
    \end{tabular}
    \vspace{1.5mm}
    \caption{Planning task where the first row defines the fluents, as well as the initial and goal states; the second row defines the available actions $\actions$; and the third row shows a plan $\plan$ that solves the task along with the states it traverses.}
    \label{tab:running_example}
    \vspace{2mm}
\end{table}

The Plan Disruption of $\plan$ is $|\{ A, B \} \triangle \{C,D,B\}| = 3$, since $A$ is present in the initial but not in the goal state, and $C$ and $D$ are true in the goal but not in the initial state.
Let us see how both compilations $\lazytask^1$ and $\eagertask^1$ approximate Plan Disruption.

Consider the following plan, which optimally solves $\lazytask^1$:
\begin{small}
    \begin{multline*}
    \plan_{\mathsf{L}}=(a_1,a_2,\mathsf{goalstate},    \mathsf{forgo}^{A},\mathsf{collect}^{B},\mathsf{forgo}^{C},\mathsf{forgo}^{D}, \mathsf{end})
\end{multline*}
\end{small}
The first two actions, $a_1$ and $a_2$, contribute $20$ to the total cost ($10+10$).
From that moment, the remainder of the plan increases the total cost depending on the plan disruption metric induced by the execution of these two actions.
The next action, $\mathsf{goalstate}$ has a cost of $0$, the same as the final $\mathsf{end}$ action.
The rest of the actions in the plan check the truth values of all the fluents \fluents in the reached goal state.
$\mathsf{forgo}^{A}$ increases the total cost by $1$, as $A$ was removed.
$\mathsf{collect}^{B}$ does not increase the total cost, as $B$ remains true in the goal state.
And both $\mathsf{forgo}^{C}$ and $\mathsf{forgo}^{D}$ increase the total cost by $1$ since $C$ and $D$ appear in the goal but not in the initial state.
Therefore, following Proposition~\ref{prop:lazy-plan-disruption}, the part of the total cost that belongs to the plan disruption metric is $3$, which equals the actual value of the metric.

On the other hand consider the following plan, which optimally solves $\eagertask^1$:
\begin{equation*}
    \plan_{\mathsf{E}} = (a_1,a_2)
\end{equation*}
The cost of $a_1$ is $10$ (the cost of the original action), plus $3$: $2$ fluents removed from \init ($A$ and $B$), and $1$ fluent added ($C$).
The cost of $a_2$ is $10$ (the cost of the original action), plus $2$: $1$ fluent removed from \init ($A$), and $1$ fluent added ($D$), since $B$ was already true in \init.
Therefore, according to Remark~\ref{rem:eager-plan-disruption}, the part of the total cost that belongs to the plan disruption metric is $5$, 
which is $2$ units higher than the actual plan disruption metric.
This difference comes from the fact that $A$ and $B$ are wrongly counted twice.
In the case of $A$ both actions delete it, but only the first time should be considered.
In the case of $B$, it is first deleted by $a_1$ and then added by $a_2$, so it should not be counted as it is true both in the initial and goal state.

In summary, while $\lazytask^\omega$ can accurately assess the disruption of a plan, it requires a polynomial increase in the number of fluents and actions.
On the other hand, $\eagertask^\omega$ may not be as accurate in certain situations, but it offers the advantage of not needing to introduce new actions, only adjusting the cost of the original ones.
\section{Evaluation}

\subsection{Experimental Setting}
\paragraph{Benchmark.} We selected all the \textsc{strips} tasks from the optimal suite of the Fast Downward~\cite{helmert2006fast} benchmark collection\footnote{https://github.com/aibasel/downward-benchmarks}.
This gives us $1847$ original tasks $\task$ divided across $66$ domains.

\paragraph{Approaches.} We evaluate the two compilations, namely Lazy Plan Disruption $\lazytask^\omega$ and Eager Plan Disruption $\eagertask^\omega$ on the above benchmark.

We experiment with $\omega=\{10^{-3},1,10^3\}$ to weight differently the importance of plan disruption wrt. sum of action costs.
With $\omega=10^{-3}$ we can expect plan disruption to just serve as a tie breaker, with sum of action costs being the main driver for solution's quality.
With $\omega=1$ we can expect plan disruption to have a similar weight than sum of action costs in the quality of a solution.
Finally, with $\omega=10^3$ we can expect plan disruption to drive the optimization, with the sum of action costs only serving as tie breaker.
We will compare the compilations to a baseline, which consists on solving the original task $\task$, where optimal plans are only defined by the sum of action costs.

\paragraph{Reproducibility.} We solve all the planning tasks ($\lazytask^\omega, \eagertask^\omega$ and $\task$) using the \textsc{seq-opt-lmcut} configuration of Fast Downward, which runs $A^*$ with the admissible \textsc{lmcut} heuristic to compute an optimal plan. 
Experiments were run on an Intel Xeon E5-2666 v3 CPU @ 2.90GHz x 8 processors with a 8GB memory bound and a time limit of 1800s.

\subsection{Results}

\paragraph{Coverage and Execution Time Overhead. }

\begin{table}[]
\setlength{\tabcolsep}{0.9pt}
\renewcommand{\arraystretch}{0.92}
    \centering
    \small
    \begin{tabular}{l|r|r|r|r|r|r|r}
         Domain (\# Problems) & $\task$ & $\eagertask^{10^{-3}}$ & $\eagertask^{1}$ & $\eagertask^{10^{3}}$ & $\lazytask^{10^{-3}}$ & $\lazytask^{1}$ & $\lazytask^{10^{3}}$ \\ \hline
        agricola-opt18 (20) & 0 & 0 & 0 & 0 & 0 & 0 & 0\\
airport (50) & 28 & 23 & 24 & 23 & 0 & 0 & 0\\
barman-opt11 (20) & 4 & 4 & 4 & 4 & 0 & 0 & 0\\
barman-opt14 (14) & 0 & 0 & 0 & 0 & 0 & 0 & 0\\
blocks (35) & 28 & 28 & 28 & 30 & 3 & 3 & 3\\
childsnack-opt14 (20) & 0 & 0 & 0 & 0 & 0 & 0 & 0\\
data-net-opt18 (20) & 12 & 12 & 13 & 19 & 0 & 0 & 0\\
depot (22) & 7 & 7 & 7 & 7 & 0 & 0 & 0\\
driverlog (20) & 13 & 14 & 14 & 13 & 1 & 1 & 1\\
elevators-opt08 (30) & 22 & 18 & 19 & 18 & 0 & 0 & 0\\
elevators-opt11 (20) & 18 & 15 & 16 & 15 & 0 & 0 & 0\\
floortile-opt11 (20) & 7 & 6 & 6 & 5 & 0 & 0 & 0\\
floortile-opt14 (20) & 6 & 5 & 5 & 2 & 0 & 0 & 0\\
freecell (80) & 15 & 15 & 15 & 11 & 0 & 0 & 0\\
ged-opt14 (20) & 15 & 13 & 20 & 20 & 0 & 0 & 0\\
grid (5) & 2 & 2 & 2 & 2 & 0 & 0 & 0\\
gripper (20) & 7 & 7 & 7 & 7 & 5 & 4 & 4\\
hiking-opt14 (20) & 9 & 9 & 9 & 10 & 3 & 0 & 0\\
logistics00 (28) & 20 & 20 & 20 & 20 & 0 & 0 & 0\\
logistics98 (35) & 6 & 6 & 6 & 6 & 0 & 0 & 0\\
miconic (150) & 141 & 141 & 141 & 141 & 36 & 35 & 35\\
movie (30) & 30 & 30 & 30 & 30 & 30 & 30 & 30\\
mprime (35) & 22 & 22 & 24 & 22 & 0 & 0 & 0\\
mystery (30) & 17 & 16 & 17 & 16 & 0 & 0 & 0\\
nomystery-opt11 (20) & 14 & 14 & 14 & 14 & 0 & 0 & 0\\
openstacks-opt08 (30) & 21 & 14 & 8 & 7 & 9 & 6 & 6\\
openstacks-opt11 (20) & 16 & 9 & 3 & 2 & 4 & 1 & 1\\
openstacks-opt14 (20) & 3 & 1 & 0 & 0 & 0 & 0 & 0\\
openstacks (30) & 7 & 7 & 7 & 7 & 5 & 5 & 5\\
org-syn-opt18 (20) & 7 & 7 & 7 & 7 & 0 & 0 & 0\\
org-syn-split-opt18 (20) & 15 & 14 & 15 & 15 & 0 & 0 & 0\\
parcprinter-08 (30) & 18 & 21 & 19 & 19 & 3 & 3 & 3\\
parcprinter-opt11 (20) & 13 & 16 & 14 & 14 & 0 & 0 & 0\\
parking-opt11 (20) & 2 & 1 & 1 & 1 & 0 & 0 & 0\\
parking-opt14 (20) & 3 & 1 & 0 & 0 & 0 & 0 & 0\\
pathways (30) & 5 & 5 & 5 & 5 & 0 & 0 & 0\\
pegsol-08 (30) & 28 & 27 & 26 & 26 & 0 & 0 & 0\\
pegsol-opt11 (20) & 18 & 17 & 16 & 16 & 0 & 0 & 0\\
petri-net-opt18 (20) & 9 & 8 & 11 & 11 & 0 & 0 & 0\\
pipes-notank (50) & 17 & 14 & 17 & 13 & 0 & 0 & 0\\
pipes-tank (50) & 12 & 8 & 9 & 8 & 0 & 0 & 0\\
psr-small (50) & 49 & 49 & 49 & 49 & 14 & 13 & 13\\
quantum-opt23 (20) & 11 & 11 & 11 & 11 & 0 & 0 & 0\\
rovers (40) & 7 & 8 & 9 & 10 & 2 & 1 & 1\\
satellite (36) & 7 & 8 & 11 & 11 & 1 & 1 & 1\\
scanalyzer-08 (30) & 15 & 9 & 9 & 7 & 3 & 3 & 3\\
scanalyzer-opt11 (20) & 12 & 6 & 6 & 4 & 1 & 1 & 1\\
snake-opt18 (20) & 6 & 4 & 6 & 4 & 0 & 0 & 0\\
sokoban-opt08 (30) & 29 & 27 & 25 & 21 & 0 & 0 & 0\\
sokoban-opt11 (20) & 20 & 20 & 19 & 17 & 0 & 0 & 0\\
spider-opt18 (20) & 11 & 7 & 6 & 6 & 0 & 0 & 0\\
storage (30) & 15 & 15 & 15 & 15 & 1 & 1 & 1\\
termes-opt18 (20) & 5 & 4 & 5 & 5 & 0 & 0 & 0\\
tetris-opt14 (17) & 6 & 3 & 4 & 3 & 0 & 0 & 0\\
tidybot-opt11 (20) & 14 & 13 & 12 & 10 & 0 & 0 & 0\\
tidybot-opt14 (20) & 8 & 7 & 7 & 5 & 0 & 0 & 0\\
tpp (30) & 6 & 6 & 6 & 6 & 3 & 3 & 3\\
transport-opt08 (30) & 11 & 11 & 11 & 11 & 3 & 3 & 3\\
transport-opt11 (20) & 6 & 6 & 6 & 6 & 0 & 0 & 0\\
transport-opt14 (20) & 6 & 6 & 6 & 6 & 0 & 0 & 0\\
trucks (30) & 10 & 10 & 10 & 10 & 0 & 0 & 0\\
visitall-opt11 (20) & 10 & 11 & 12 & 10 & 7 & 6 & 6\\
visitall-opt14 (20) & 5 & 6 & 6 & 5 & 1 & 0 & 0\\
woodworking-opt08 (30) & 17 & 17 & 18 & 20 & 1 & 1 & 1\\
woodworking-opt11 (20) & 12 & 11 & 12 & 14 & 0 & 0 & 0\\
zenotravel (20) & 13 & 12 & 13 & 11 & 2 & 2 & 2\\ \hline
Total (1847) & 948  & 884  & 893  & 863  & 138  & 123  & 123 

    \end{tabular}
    \caption{Number of problems solved by each compilation.}
    \label{tab:coverage}
\end{table}

First, we aim to understand the difficulty of solving our compiled tasks compared to the original task. 
Table~\ref{tab:coverage} presents the coverage of \textsc{seq-opt-lmcut} when solving both the compiled and original tasks across all domains and problems. 
As expected, solving the standard planning task, where only cost optimization is considered, is easier, allowing the planner to find cost-optimal plans for 948 instances. 
The eager compilation $\eagertask^\omega$, which merely updates the cost of the original actions, solves nearly the same number of tasks, with 893 instances when $\omega=1$. 
In contrast, the lazy compilation $\lazytask^\omega$, which requires introducing additional actions and transforming the original problem into an oversubscription planning task, solves only 138 tasks. 
This clearly indicates that the lazy compilation results in tasks that are more complex than both the eager compilation and the original tasks.

In order to further understand the overhead introduced by our compilations, we compare the execution time $T$ needed by the planner to solve the new tasks $X$ versus the time needed to solve the standard planning task $\task$.
We refer to this as the \emph{time overhead factor}, and formally define it as $\frac{T(X)}{T(\task)}$.
This execution time includes both the time needed to translate (and ground) the task and the solving time.
To make the comparison fair, we only consider the $123$ problems that are commonly solved by all the approaches.
Figure~\ref{fig:execution_time_overhead} shows this analysis as a set of violinplots, which represent the distribution of these factors in log scale for each compilation.
As we can see, the eager compilations hardly introduce any overhead compared to solving the original task.
Most of the tasks can be solved in the same amount of time, with around $50$ tasks that are faster to solve under the new cost's setting.
On the other hand, solving the lazy tasks can require execution times that are one to five orders of magnitude longer.

\begin{figure}
    \centering
    \includegraphics[width=1\columnwidth]{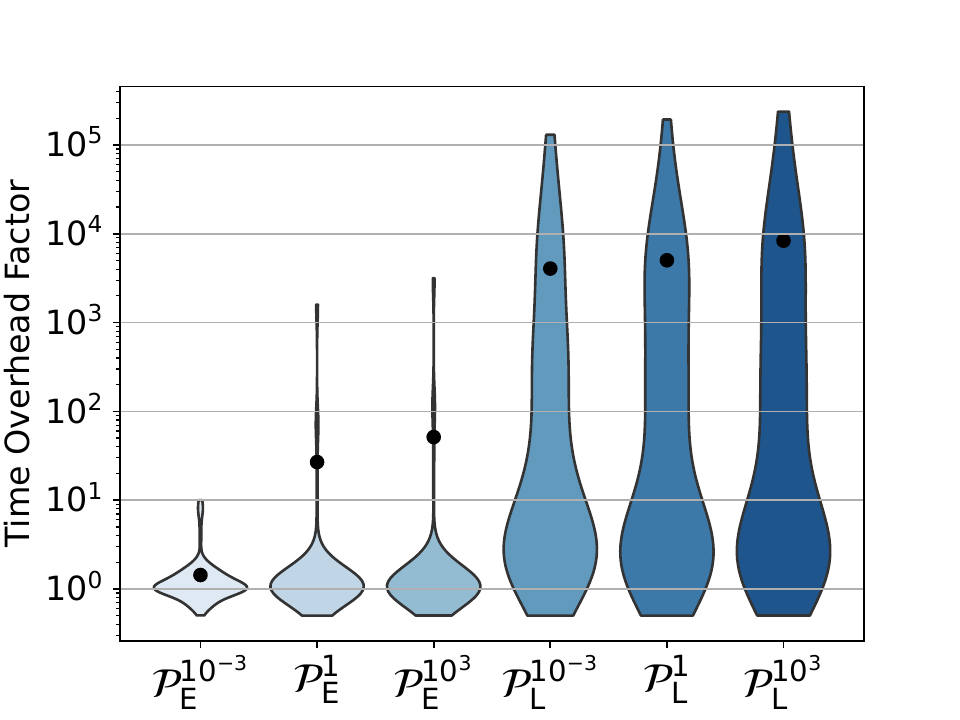}
    \caption{Distribution of the execution time overhead factor $\frac{T(X)}{T(\task)}$ for each compilation $X$. Black dots represent the average.}
    \label{fig:execution_time_overhead}
\end{figure}

\paragraph{Is Solving $\task$ Good Enough?}

\begin{figure*}
    \centering
    \begin{subfigure}[h]{0.33\textwidth}\centering
    \includegraphics[width=\textwidth]{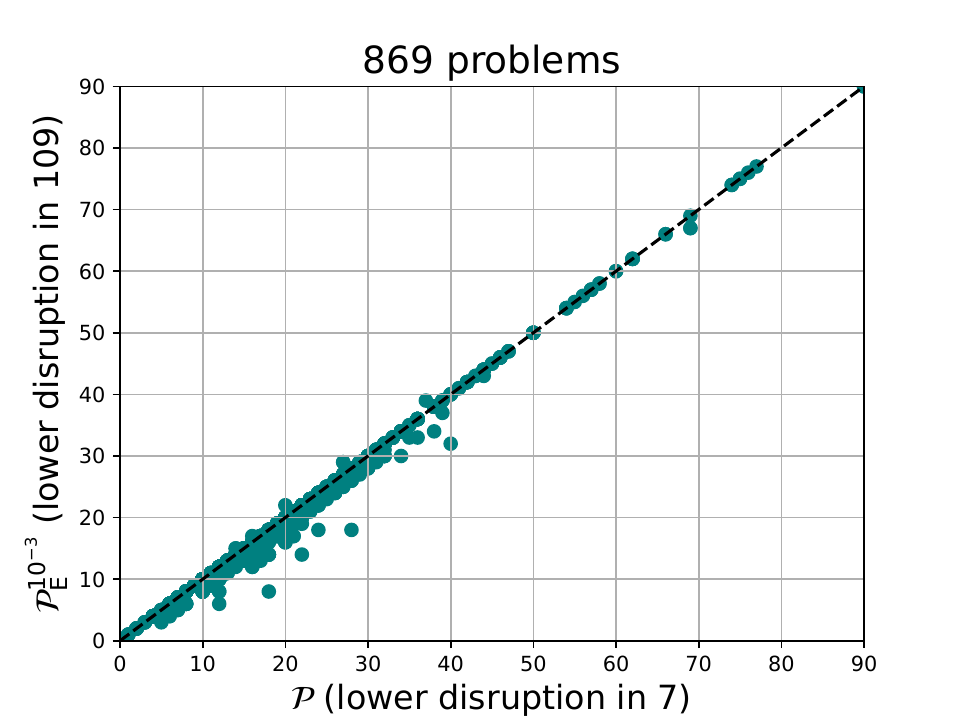}
    \caption{$\eagertask^{10^{-3}}$}
    \label{fig:a}
    \end{subfigure}
    \hfill
    \begin{subfigure}[h]{0.33\textwidth}\centering
    \includegraphics[width=\textwidth]{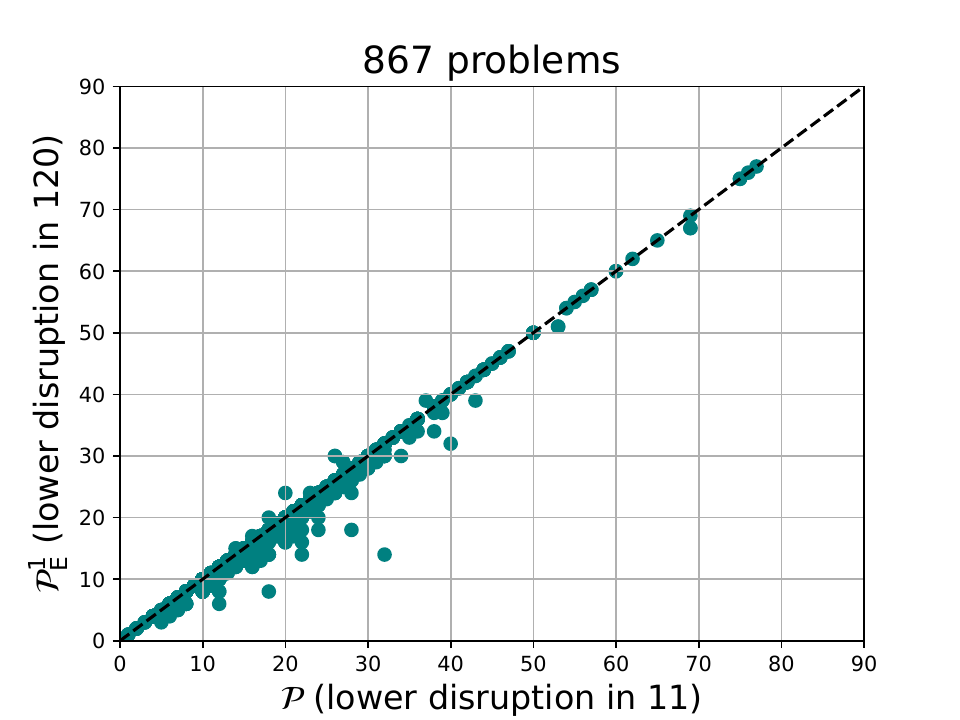}
    \caption{$\eagertask^{1}$}
    \label{fig:b}
    \end{subfigure}
    \hfill
    \begin{subfigure}[h]{0.33\textwidth}\centering
    \includegraphics[width=\textwidth]{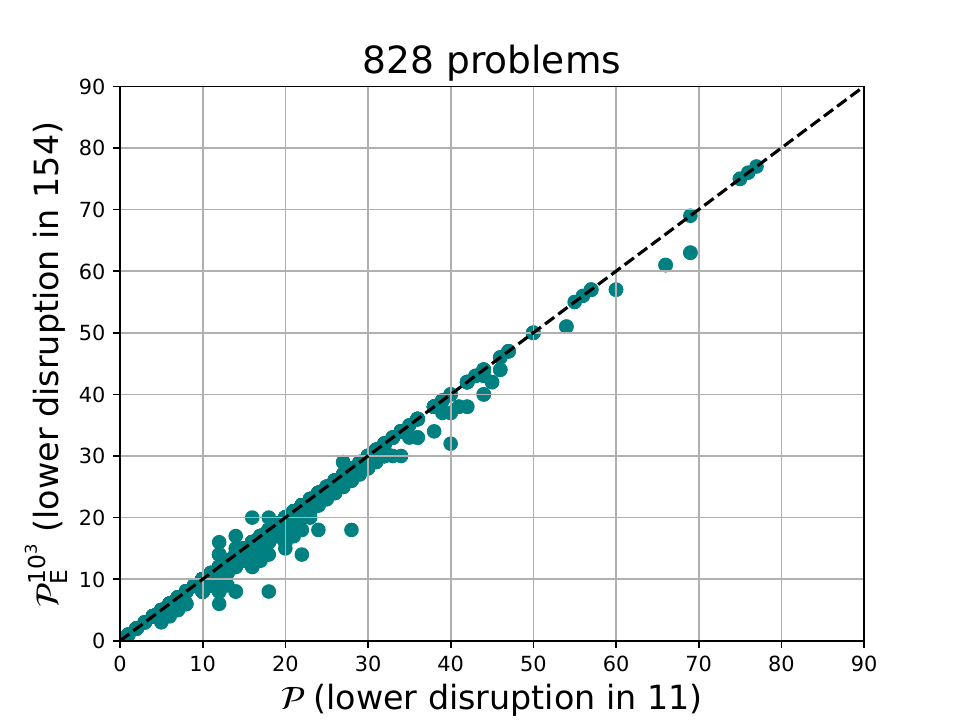}
    \caption{$\eagertask^{10^{3}}$}
    \label{fig:c}
    \end{subfigure}
    \bigskip

    \vspace{-5mm}
    
    \begin{subfigure}[t]{0.33\textwidth}\centering
    \includegraphics[width=\textwidth]{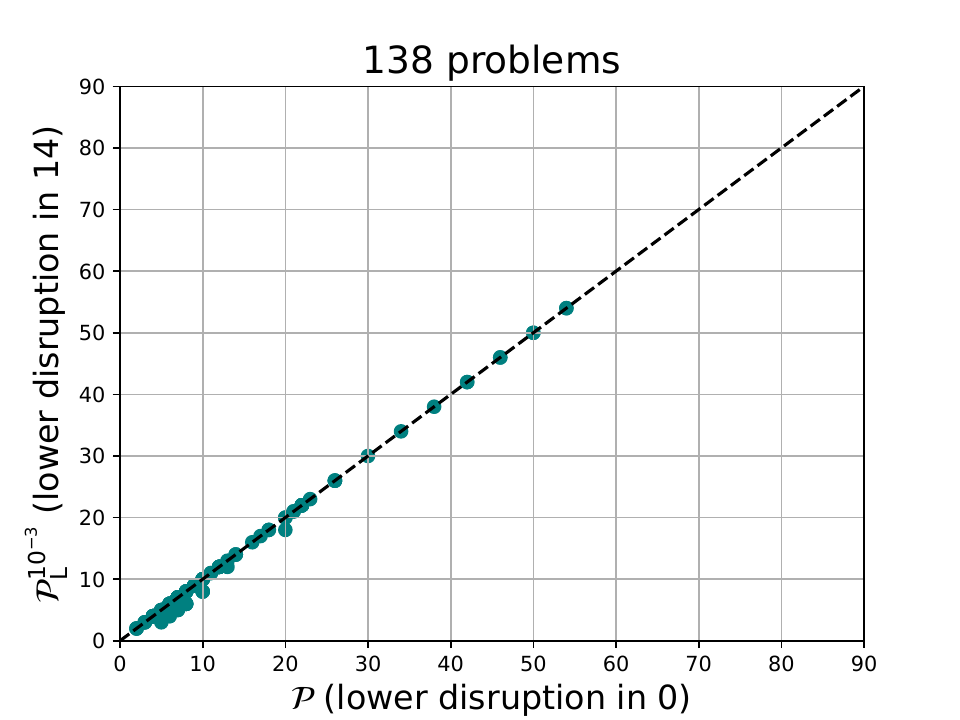}
    \caption{$\lazytask^{10^{-3}}$ }
    \label{fig:a_time}
    \end{subfigure}
    \hfill
    \begin{subfigure}[t]{0.33\textwidth}\centering
    \includegraphics[width=\textwidth]{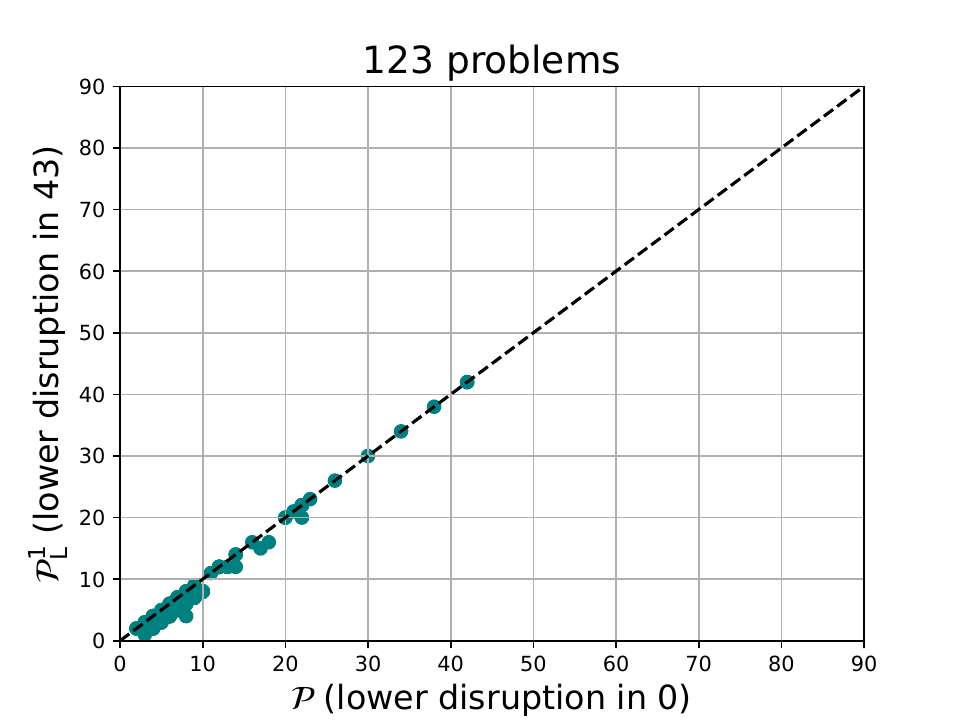}
    \caption{$\lazytask^{1}$ }
    \label{fig:b_time}
    \end{subfigure}
    \hfill
    \begin{subfigure}[t]{0.33\textwidth}\centering
    \includegraphics[width=\textwidth]{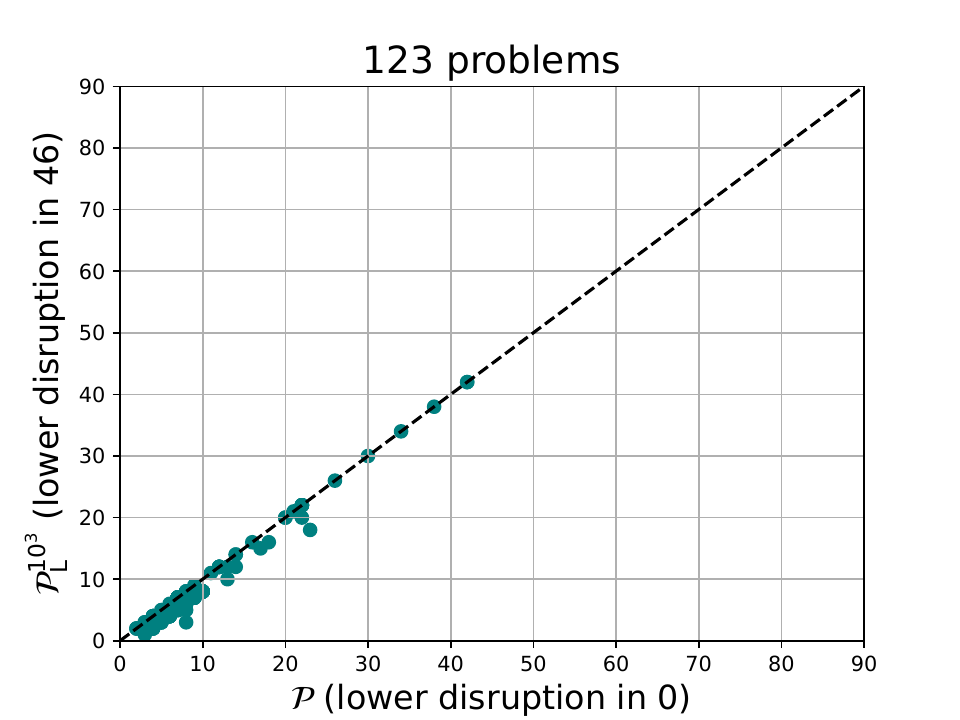}
    \caption{$\lazytask^{10^{3}}$ }
    \label{fig:c_time}
    \end{subfigure}
    \bigskip
    \vspace{-5mm}
    \caption{Plan disruption for optimally solved compiled tasks ($y$-axis) versus original tasks ($x$-axis). Figure titles show the number of problems plotted, i.e., those solved by both the compilation and the original task.} 
    \label{fig:big_figure}
\end{figure*}

Given the coverage and execution time results, one might question whether solving the standard planning task $\task$ already produces plans with minimal disruption.
To challenge this hypothesis, we compare the disruption metrics of the plans obtained from solving each compilation with those of the plans derived from solving the original task.
These results are shown in Figure~\ref{fig:big_figure}.
The title of each subfigure indicates the number of problems represented in the plot, i.e., those commonly solved by the given compilation ($y$-axis) and the original task ($x$-axis).
Points below the diagonal indicate that the plan returned after solving the compilation has a lower (better) plan disruption value.

As we can see, most points lie along the diagonal regardless of the compilation, indicating that the cost-optimal plan has the same disruption value as the compilation that explicitly optimizes this metric. 
In both the eager and lazy compilations, higher $\omega$ values result in more plans with better disruption values compared to those produced by the original task. 
In the eager compilation $\eagertask^\omega$, this difference increases from $109/869=12.5\%$ of the plans with $\omega=10^{-3}$ to $154/828=18.5\%$ with $\omega=10^3$. 
A similar trend is seen in the lazy compilation $\lazytask^\omega$, where the percentage rises from $10\%$ to $37\%$. 
As discussed in Section 5, the eager compilation, although serving as a proxy for plan disruption, may introduce noise in its computation. 
In some cases, this can lead to plans that inaccurately assess the correct disruption metric, resulting in the cost-optimal plan having lower disruption than the one derived from solving the eager compilation. 
Conversely, the lazy compilation, which accurately computes the plan disruption, never returns a plan with higher disruption than when solving the original task.

\paragraph{Trading-off Plan Disruption and Cost.}

The previous results suggest that there is not much variability in the plan disruption of the plans that solve the planning tasks in the benchmark.
Although this is the general trend, there are some tasks where we can observe a trade-off between optimizing plan cost and disruption.
This is exemplified in the \textsc{satellite} task depicted in Figure~\ref{fig:tradeoff}, which shows the cost ($y$-axis) and disruption ($x$-axis) of the plans returned by each compilation.
As we can see, the plan we get after solving the original task $\task$ (blue dot) is cost-optimal ($9$) but has the highest disruption value, making $8$ changes to the initial state in order to achieve the goal.
All the compilations manage to get plans with lower disruption, with some of them achieving the same cost as to when solving $\task$.
On the other extreme of the spectrum we have the lazy compilation with $\omega=10^3$, which is able to return a plan with a plan disruption of $3$ at the expense of increasing the cost of the plan from $9$ to $12$.
This result clearly demonstrates that when the original task offers sufficient diversity in the plans that solve it, our proposed compilations can effectively balance plan cost and disruption, yielding plans that prioritize each objective differently.

\begin{figure}
    \centering
    \includegraphics[width=1\columnwidth]{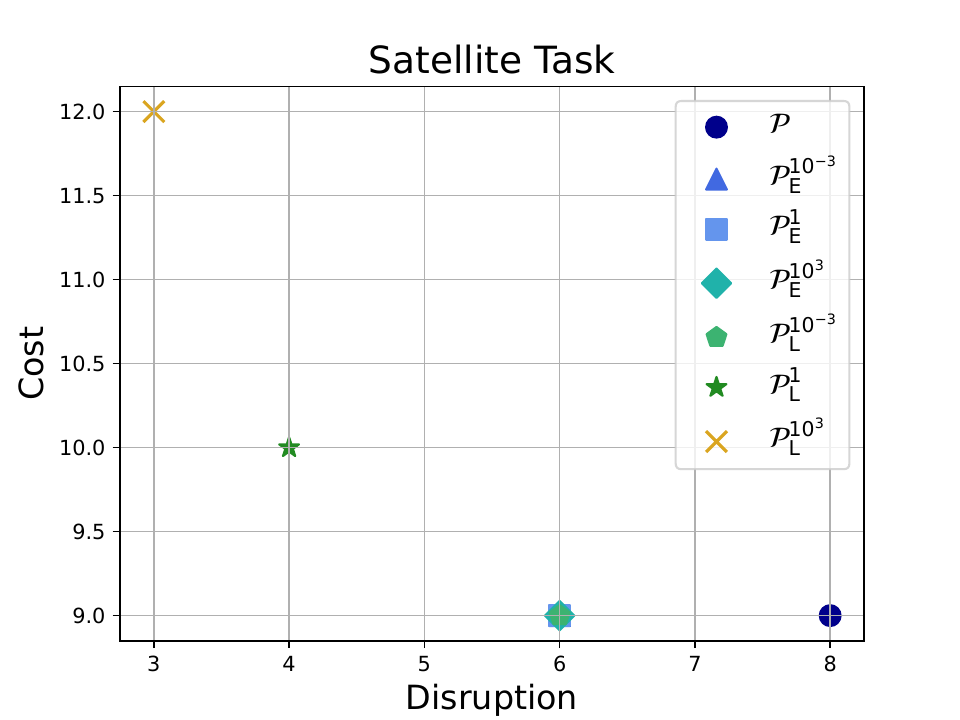}
    \caption{Plan Cost ($y$-axis) vs Plan Disruption ($x$-axis) of the plans that solve a \textsc{satellite} task under the different compilations.}
    \label{fig:tradeoff}
\end{figure}

\section{Conclusions and Future Work}

In this paper, we introduce an objective that may be relevant to many planning applications: finding plans that minimally alter the initial state to achieve the goals. 
We term this concept \emph{plan disruption} and propose two compilations that jointly optimize plan cost and disruption. 
Experimental results from a comprehensive benchmark indicate that, although most planning tasks exhibit limited variability in plan disruption, our compilations effectively balance both objectives in tasks where there is potential for improving plan disruption. 
The eager compilation scales similarly to the standard planning task and effectively minimizes plan disruption. 
However, because it minimizes a proxy for plan disruption rather than the actual metric, it can occasionally produce plans with higher disruption values than those obtained from solving the original task. 
Conversely, the lazy compilation generally requires several orders of magnitude more execution time, making it unsuitable for larger planning tasks. 
Despite this, the additional time investment results in plans with the lowest possible plan dispersion, making it the preferred compilation for smaller planning tasks.

In this work we solely focused on computing optimal solutions for all the tasks, revealing that some of our compilations face scalability challenges.
In future work we would like to solve the reformulated tasks using satisficing planners to study the trade-off between scalability and suboptimality.
In this paper, we treated all the atoms in the initial state equally.
We would also like to allow users to specify the importance of keeping the truth value of each atom, either by explicitly assigning weights or implicitly by providing input plans that reflect their preferences~\cite{morales2024learning}.

\section*{Disclaimer}
This paper was prepared for informational purposes by the Artificial Intelligence Research group of JPMorgan Chase \& Co. and its affiliates ("JP Morgan'') and is not a product of the Research Department of JP Morgan. JP Morgan makes no representation and warranty whatsoever and disclaims all liability, for the completeness, accuracy or reliability of the information contained herein. This document is not intended as investment research or investment advice, or a recommendation, offer or solicitation for the purchase or sale of any security, financial instrument, financial product or service, or to be used in any way for evaluating the merits of participating in any transaction, and shall not constitute a solicitation under any jurisdiction or to any person, if such solicitation under such jurisdiction or to such person would be unlawful.

\bibliography{aaai2026}

\end{document}